\documentclass[letterpaper]{article} % DO NOT CHANGE THIS
\usepackage{aaai24}  % DO NOT CHANGE THIS
\usepackage{times}  % DO NOT CHANGE THIS
\usepackage{helvet}  % DO NOT CHANGE THIS
\usepackage{courier}  % DO NOT CHANGE THIS
\usepackage[hyphens]{url}  % DO NOT CHANGE THIS
\usepackage{graphicx} % DO NOT CHANGE THIS
\usepackage{natbib}  % DO NOT CHANGE THIS AND DO NOT ADD ANY OPTIONS TO IT
\usepackage{caption} % DO NOT CHANGE THIS AND DO NOT ADD ANY OPTIONS TO IT
\usepackage{algorithm}
\usepackage{algorithmic}
\usepackage{multirow}
\usepackage{amsmath,amscd,amsbsy,amssymb,latexsym,url,bm,amsthm}

\newtheorem{remark}{Remark}

\newtheorem{theorem}{Theorem}

\newtheorem{lemma}{Lemma}

\newtheorem{definition}{Definition}
\newtheorem{example}{Example}

\usepackage[utf8]{inputenc} % allow utf-8 input
\usepackage[T1]{fontenc}    % use 8-bit T1 fonts
\usepackage{url}            % simple URL typesetting
\usepackage{booktabs}       % professional-quality tables
\usepackage{amsfonts}       % blackboard math symbols
\usepackage{nicefrac}       % compact symbols for 1/2, etc.
\usepackage{microtype}      % microtypography
\usepackage{algorithm,algorithmic}

\usepackage{amsmath,amscd,amsbsy,amssymb,latexsym,url,bm,amsthm}
\allowdisplaybreaks
\usepackage{cleveref}
\usepackage{verbatim}
\usepackage{color}
\usepackage{dsfont}

\newcommand{\set}[1]{\left\{ #1 \right\}}

\newcommand{\cF}{\mathcal{F}}

\newcommand{\cN}{\mathcal{N}}

\newcommand{\cK}{\mathcal{K}}

\newcommand{\true}{\mathrm{True}}
\newcommand{\false}{\mathrm{False}}

\newcommand{\ucb}{{\mathrm{UCB}}}
\newcommand{\lcb}{{\mathrm{LCB}}}
\newcommand{\opt}{{\mathrm{opt}}}

\newcommand{\abs}[1]{\left| #1 \right|}
\newcommand{\bOne}[1]{\mathds{1} \! \left\{#1\right\}}
\newcommand{\bracket}[1]{\left(#1\right)}

\newcommand{\EE}[1]{\mathbb{E} \left[#1\right]}
\newcommand{\PP}[1]{\mathbb{P} \left(#1\right)}

\usepackage{tikz}
\usetikzlibrary{shapes,shapes.misc,positioning}
\usetikzlibrary{patterns,arrows,decorations.pathreplacing}

\DeclareMathOperator*{\argmin}{argmin}

\mathchardef\mhyphen="2D
\newcommand{\ch}{{\tt Ch}}

\usepackage{newfloat}
\usepackage{listings}
\DeclareCaptionStyle{ruled}{labelfont=normalfont,labelsep=colon,strut=off} % DO NOT CHANGE THIS
\floatstyle{ruled}
\newfloat{listing}{tb}{lst}{}
\floatname{listing}{Listing}
\title{
Improved Bandits in Many-to-One Matching Markets with Incentive Compatibility
}

\author {
    Fang Kong, Shuai Li\thanks{Corresponding author.}
}
\affiliations {
    John Hopcroft Center for Computer Science, Shanghai Jiao Tong University \\
    \{fangkong, shuaili8\}@sjtu.edu.cn
}

\usepackage{bibentry}
\usepackage{color}

\begin{document}

\maketitle

\begin{abstract}

Two-sided matching markets have been widely studied in the literature due to their rich applications. Since participants are usually uncertain about their preferences, online algorithms have recently been adopted to learn them through iterative interactions. An existing work initiates the study of this problem in a many-to-one setting with responsiveness. However, their results are far from optimal and lack guarantees of incentive compatibility. We first extend an existing algorithm for the one-to-one setting to this more general setting and show it achieves a near-optimal bound for player-optimal regret. Nevertheless, due to the substantial requirement for collaboration, a single player's deviation could lead to a huge increase in its own cumulative rewards and a linear regret for others. In this paper, we aim to enhance the regret bound in many-to-one markets while ensuring incentive compatibility. We first propose the adaptively explore-then-deferred-acceptance (AETDA) algorithm for responsiveness setting and derive an upper bound for player-optimal stable regret while demonstrating its guarantee of incentive compatibility. To the best of our knowledge, it constitutes the first polynomial player-optimal guarantee in matching markets that offers such robust assurances without known $\Delta$, where $\Delta$ is some preference gap among players and arms. 
We also consider broader substitutable preferences, one of the most general conditions to ensure the existence of a stable matching and cover responsiveness. We devise an online DA (ODA) algorithm and establish an upper bound for the player-pessimal stable regret for this setting. 
\end{abstract}

%!TEX root =  main.tex
\section{Introduction}\label{sec:intro}

The problem of two-sided matching markets has been studied for a long history due to its wide range of applications in real life including the labor market and college admission \citep{gale1962college,roth1984evolution,roth1992two,abdulkadirouglu1999house,epple2006admission,fu2014equilibrium}. 
There are two sides of market participants, e.g., employers and workers in the labor market, and each side has a preference ranking over the other side. 
The matching reflects the bilateral nature of exchange in the market. For example, a worker works for an employer and the employer employs this worker.  
Stability is a key concept describing the equilibrium of a matching, which ensures the current bilateral exchange cannot be easily broken.  
A rich line of works study how to find a stable matching in the market \citep{gale1962college,kelso1982job,roth1984evolution,roth1992two,erdil2019efficiency}.
However, all of them assume the preferences of market participants are known \emph{a priori}, which may not be satisfied in practice. 
For example in labor markets, workers usually have unknown preferences over employers since they do not know whether they like the task type or the employer. 
With the emergence of online marketplaces such as online labor market Upwork and crowdsourcing platform Amazon Mechanical Turk where employers have numerous similar tasks to delegate, workers are able to learn the uncertain preferences during the iterative matching process with employers through these tasks.

Multi-armed bandit (MAB) is a core problem that characterizes the learning process during iterative interactions when faced with uncertainty  \citep{auer2002finite,lattimore2020bandit}. 
There are also two sides of agents: a player on one side and $K$ arms on the other side. The player has unknown preferences over arms. At each time, it selects an arm and receives a reward. 
The player's objective is to maximize the cumulative reward over a specified horizon. 
To better measure the performance of the player's strategy, an equivalent objective of minimizing the cumulative regret is widely studied, which is defined as the cumulative difference between the reward of the optimal arm and that of the selected arms.

Recently, a rich line of works study the bandit learning problem in matching markets where more than one player and arms exist. 
These works study the case where players have unknown preferences over arms and arms can determine their preferences over players based on some known utilities such as the profile of workers in online labor markets.
To characterize the stability of the learned matching, the objective of stable regret is adopted and studied \citep{das2005two,liu2020competing,liu2021bandit,sankararaman2021dominate,basu21beyond,kong2022thompson,zhang2022matching,kong2023player,wang2022bandit}.
Previous works mainly focus on two types of objectives: the player-optimal stable regret and the player-pessimal stable regret. 
The former is defined as the cumulative difference between the reward of the arm in the players' most preferred stable matching and the accumulated reward by the player. 
The latter is defined compared with the reward of the arm in the players' least preferred stable matching. 
\citet{liu2020competing} first study the centralized version where a central platform assigns an allocation of arms to players in each round and provide theoretical guarantees. 
Since such a platform may not always exist in real applications, the following works mainly focus on the decentralized setting where each player makes her own decision \citep{liu2021bandit,sankararaman2021dominate,basu21beyond,kong2022thompson,maheshwari2022decentralized}. 
These works only achieve guarantees on the player-pessimal stable regret \citep{liu2021bandit,kong2023player} or study special markets where unique stable matching exists. Until recently, \citet{zhang2022matching} and \citet{kong2023player} independently propose algorithms that can reach the player-optimal stable matching.

All of the above works study the one-to-one matching markets where each player proposes to one arm at a time and each arm could accept at most one player. 
The many-to-one setting is more general and common in real life such as in labor markets where an employer usually has a certain quota and can recruit a group of workers \citep{roth1984stability,roth1992two,abdulkadirouglu2005college,che2019stable}. 
\citet{wang2022bandit} initiate the study in many-to-one markets by considering that arms have responsive preferences. 
However, their algorithm is only able to achieve player-pessimal stable matching and lacks guarantees on incentive compatibility.  
Incentive compatibility is a crucial property in multi-player systems as it ensures players are incentivized to act in ways that align with desired system outcomes, thereby promoting cooperation and efficiency rather than encouraging competitive or destructive behaviors. Deriving algorithms that can achieve better regret and enjoy guarantees on this property is important in matching markets.

In this paper, we aim to provide algorithms with improved regret guarantee and incentive compatibility for many-to-one markets. 
For generality, we also study the decentralized setting.
We propose an adaptive explore-then-DA (AETDA) algorithm for markets with responsive preferences and derive $O(N\min\set{N,K}C\log T/\Delta^2)$  upper bound for the player-optimal stable regret as well as a guarantee of incentive compatibility, where $N$ is the number of players, $K$ is the number of arms, $C$ is arms' total capacities, $T$ is the horizon, and $\Delta$ is some preference gap among players and arms. 
To the best of our knowledge, it is the first guarantee for the player-optimal regret in decentralized many-to-one markets and is also the first that simultaneously enjoys such robust assurance in one-to-one markets without known $\Delta$. 
Since arms preferences may possess a combinatorial structure which may not be well characterized by responsiveness, we also consider a more general setting with \textit{substitutability}~\citep{roth1992two}, one of the most generally known conditions to ensure the existence of a stable matching and naturally holds under responsiveness \citep{roth1992two,abdulkadirouglu2005college}.
We design an online deferred acceptance (ODA) algorithm for this more general setting and prove that the regret against the player-pessimal stable matching is bounded by $O(NK\log T/\Delta^2)$.
Table \ref{table:comparison} provides a comprehensive comparison between our work and related results. 
% As compared in Table \ref{table:comparison}, this result not only works under a more general setting but also achieves a great advantage over \citet{wang2022bandit}. 

\begin{table*}[htb!]
\centering
\begin{tabular}{lll}
\toprule 
  & Regret bound       & Setting                     \\\hline
\rule{0pt}{13pt}\multirow{2}{*}{\citet{liu2020competing}} & $\displaystyle O\bracket{{K\log T}/{\Delta^2}} *\#$  & one-one, known $\Delta$, incentive, $\mathrm{gap}_1$\\
% , \tiny{}
& $\displaystyle O\bracket{{NK\log T/\Delta^2}}\#$              & one-one, $\mathrm{gap}_5$ \\\hline
\rule{0pt}{20pt}\citet{liu2021bandit}                                                                                 & $\displaystyle O\bracket{\frac{N^5K^2\log^2 T}{\varepsilon^{N^{4}}\Delta^2}}$                                                       & one-one, $\mathrm{gap}_5$                        \\ \hline
\rule{0pt}{12pt}\multirow{2}{*}{\citet{sankararaman2021dominate} }& $\displaystyle O\bracket{{NK\log T}/{\Delta^2}}$ & {one-one (serial dictatorship), -}   \\ &
              $\displaystyle\Omega\bracket{{N\log T}/{\Delta^2}}$                &       incentive, $\mathrm{gap}_1$               \\\hline
\rule{0pt}{13pt}\multirow{2}{*}{\citet{basu21beyond}} & $\displaystyle O\bracket{K\log^{1+\varepsilon} T + 
2^{(\frac{1}{\Delta^2})^{\frac{1}{\varepsilon}}} } *$                 & one-one, $\mathrm{gap}_5$ \\
                       &  $\displaystyle O\bracket{{NK\log T}/{\Delta^2}}$                &   one-one (uniqueness), $\mathrm{gap}_1$                                     \\\hline
\rule{0pt}{15pt}\citet{maheshwari2022decentralized}                                                                                 &   $O\bracket{C'NK\log T/\Delta^2}$                                                     & one-one ($\alpha$-reducible), $\mathrm{gap}_1$                         \\ \hline
\rule{0pt}{20pt}\citet{kong2022thompson}                                                                                 & $\displaystyle O\bracket{\frac{N^5K^2\log^2 T}{\varepsilon^{N^{4}}\Delta^2}}$                                                       & one-one, $\mathrm{gap}_5$                       \\ \hline
\rule{0pt}{15pt}\citet{zhang2022matching}                                                                                 & $\displaystyle O\bracket{K\log T/\Delta^2}*$                                                       & one-one, $\mathrm{gap}_5$                    \\ \hline
\rule{0pt}{12pt}\multirow{2}{*}{\citet{kong2023player} }& \multirow{2}{*}{$\displaystyle O\bracket{{K\log T}/{\Delta^2}}*$} & one-one, $\mathrm{gap}_4$   \\ &   &  responsiveness \textbf{(ours)}, $\mathrm{gap}_4$                         \\\hline
\rule{0pt}{12pt}\multirow{3}{*}{\citet{wang2022bandit}}& $\displaystyle O\bracket{{K\log T}/{\Delta^2}}*\#$ &  responsiveness, known $\Delta$, $\mathrm{gap}_1$                     \\
  & $\displaystyle O\bracket{{NK^3\log T}/{\Delta^2}}\#$  & responsiveness, $\mathrm{gap}_5$ \\
& {$\displaystyle O\bracket{\frac{N^5K^2\log^2 T}{\varepsilon^{N^{4}}\Delta^2}}$} & responsiveness, $\mathrm{gap}_5$                        \\ \hline
\rule{0pt}{12pt}\multirow{2}{*}{\textbf{Ours}} & {$\displaystyle O\bracket{\frac{N\min\set{N,K}C\log T}{\Delta^2}}*$} & responsiveness, incentive, $\mathrm{gap}_3$ 
\\ & $\displaystyle O\bracket{{NK\log T}/{\Delta^2}}$ &  substitutability, $\mathrm{gap}_2$
                      \\

          \bottomrule 
\end{tabular}
\caption{Comparisons of settings and regret bounds with most related works. $*$ represents the player-optimal stable regret and bounds without labeling $*$ are for player-pessimal stable regret, 
$\#$ represents the centralized setting.  
$N,K,\Delta,C, \varepsilon, C'$ are the number of players and arms, some preference gap among players and arms, the total capacities of all arms under the responsiveness condition, the hyper-parameter of algorithms which can be very small, and the parameter related to the unique stable matching condition which can grow exponentially in $N$,  respectively.   
`Incentive' means that there is a guarantee for incentive compatibility. 
The definition of $\Delta$ requires particular care in different results. 
It may be defined as the minimum preference gap between the player-optimal stable arm and the next arm among all players (labeled as $\mathrm{gap}_1$); the minimum preference gap between the player-pessimal stable arm and the next arm among all players (labeled as $\mathrm{gap}_2$); the minimum preference gap between any arms that have higher ranking than the arm after the player-optimal stable arm (labeled as $\mathrm{gap}_3$); the minimum preference gap between any arms that have higher ranking than $\min\set{N+1,K}$ (labeled as $\mathrm{gap}_4$); and the minimum preference gap between any different arms among all players (labeled as $\mathrm{gap}_5$). Based on the fact that the player-optimal stable arm must be the first $\min\set{N,K}$-ranked (proved in Appendix), it holds that $\mathrm{gap}_1>\mathrm{gap}_3>\mathrm{gap}_4>\mathrm{gap}_5$, and $\mathrm{gap}_2>\mathrm{gap}_5$. 
} 
\label{table:comparison}
\end{table*}

% \fang{introduce stricted substi; \textit{restricted} footnote refer to; delete $,$}

% Specific to the responsiveness setting, we also propose a more efficient explore-then-DA (ETDA) algorithm and provide an $O(K\log T/\Delta^2)$ upper bound for the player-optimal stable regret. 
% This is the first player-optimal guarantee for decentralized many-to-one markets and achieves the same order as the state-of-the-art result in the one-to-one setting. 
% We also demonstrate the convergence of our algorithms and their advantage over baselines in a series of experiments.

%!TEX root =  main.tex
\section{Related Work}\label{sec:related}

The matching market model characterizes many applications such as labor market \citep{roth1984evolution}, house allocation \citep{abdulkadirouglu1999house}, college admission and marriage problems \citep{gale1962college}, among which the many-to-one setting is very common and widely studied \citep{roth1992two}. 
Responsiveness and substitutability are the most generally known conditions to guarantee the existence of a stable matching \citep{kelso1982job,roth1984stability,abdulkadirouglu2005college} and the deferred acceptance (DA) algorithm is a classical offline algorithm to find a stable matching under this property \citep{kelso1982job,roth1984stability}.

For simplicity, we refer to the setting where one-side participants (players) have unknown preferences as the online setting. 
This line of works relies on the technique of MAB, a classical online learning framework with a single player and $K$ arms \citep{lattimore2020bandit}. 
The explore-then-commit (ETC) \citep{garivier2016explore}, upper confidence bound (UCB) \citep{auer2002finite}, Thompson sampling (TS) \citep{agrawal2012analysis} and elimination \citep{auer2010ucbelimination} algorithms are common strategies to obtain $O(K\log T/\Delta)$ regret where $\Delta$ is the minimum suboptimality gap among arms. 

Multiple-player MAB (MP-MAB) generalizes the standard MAB problem by considering more than one player in the environment. 
In this setting, each player selects an arm at each time and a player will receive nothing if it collides with others by selecting the same arm. 
The MP-MAB problem has been studied in both homogeneous and heterogeneous settings. The former assumes different players share the same preference over arms \citep{rosenski2016multi,bubeck2021cooperative} and the latter allows players to have different preferences \citep{bistritz2018distributed,shi2021heterogeneous}. 
Both settings aim to minimize the collective cumulative regret of all players. 
% which is compared with the maximum collective reward received by all players. 

The matching market problem is different from the above MP-MAB framework by considering that each arm also has a preference ranking over players.  
When multiple players select one arm, the player preferred most by the arm would not be collided and would gain a reward.   
The objective in this setting is to learn a stable matching and minimize the stable regret for players. 
\citet{das2005two} first introduce the bandit learning problem in one-to-one matching markets and explore the empirical performances of the proposed algorithms.  \citet{liu2020competing} initiate the theoretical study on this problem. They first consider the centralized setting where a central platform assigns allocations to players in each round. 
Later,  \citet{sankararaman2021dominate}, \citet{basu21beyond} and \citet{maheshwari2022decentralized} successively study this setting in a decentralized manner where players make their own decisions without a central platform. 
These works additionally assume the preferences of participants satisfy some constraints to ensure the uniqueness of the stable matching.  
For a general decentralized one-to-one market, \citet{liu2021bandit} and \citet{kong2022thompson} propose UCB and TS-type algorithms, respectively. However, they only derive guarantees on the player-pessimal stable regret. 
% with guarantees for player-pessimal stable regret 
Until recently, the theoretical analysis for the player-optimal stable regret has been derived by \citet{zhang2022matching} and \citet{kong2023player} independently. 

Due to the generality when modeling real applications, \citet{wang2022bandit} start to study the bandit problem in many-to-one settings. They assume arms have responsive preferences and derive algorithms in both centralized and decentralized settings. For the decentralized setting, only an upper bound for the player-pessimal stable regret is provided. 
Table \ref{table:comparison} compares our results with the most related works for matching markets. 
As shown in the table, our results not only work under a more general setting but also achieve a great advantage over \citet{wang2022bandit}.  
% and the upper bound is far from optimal. Please see Table \ref{table:comparison} for a comprehensive comparison among these works. 
% with the upper bound $O(N^5K^2\log^2 T/(\varepsilon^{N^4}\Delta^2))$ where $\varepsilon\in(0,1)$ is a hyper-parameter. 
% \shuai{introduce table 1 somewhere. table 1 does not introduce N,K,Delta}
%!TEX root =  main.tex
\section{Setting}\label{sec:setting}

The two-sided market consists of $N$ players and $K$ arms. Denote the player and the arm set as $\cN = \set{p_1,p_2,\ldots,p_N}$ and $\cK=\set{a_1,a_2,\ldots,a_K}$, respectively. 
Just as in common applications such as the online labor market, players have preferences over individual arms.  
The relative preference of player $p_i$ for arm $a_j$ can be quantified by a real value $\mu_{i,j}\in (0,1]$, which is unknown and needs to be learned during interactions with arms.
For each player $p_i$, we assume $\mu_{i,j}\neq\mu_{i,j'}$ for distinct arms $a_j,a_{j'}$ as in previous works \cite{kelso1982job,roth1984stability,liu2020competing,liu2021bandit,kong2023player,wang2022bandit}. 
And $\mu_{i,j}>\mu_{i,j'}$ implies that player $p_i$ prefers $a_j$ to $a_{j'}$. 
For the other side of participants, arms are usually certain of their preferences for players based on some known utilities, e.g., the profiles of workers in the online labor markets scenario. In many-to-one markets, when faced with a set $P \subseteq \cN$ of players, the arm can determine which subset of $P$ it prefers most. Denote $\ch_j(P)$ as this choice of arm $j$ when faced with $P$. Then for any $P' \subseteq P$, arm $a_j$ prefers $\ch_j(P)$ to $P'$.

At each round $t=1,2,\ldots$, each player $p_i \in \cN$ proposes to an arm $A_i(t) \in \cK$. Let $A^{-1}_j(t) = \set{p_i: A_i(t)=a_j}$ be the set of players who propose to $a_j$. When faced with the player set $A^{-1}_j(t)$, arm $a_j$ only accepts its most preferred subset $\ch_j(A^{-1}_j(t))$ and would reject others. 
Once $p_i$ is successfully accepted by arm $A_i(t)$, it receives a utility gain $X_{i,A_i(t)}(t)$, which is a $1$-subgaussian random variable with expectation $\mu_{i,A_i(t)}$. 
Otherwise, it receives  $X_{i,A_i(t)}(t)=0$. 
We further denote $\bar{A}_i(t)$ as $p_i$'s matched arm at round $t$. Specifically, $\bar{A}_i(t)=A_i(t)$ if $p_i$ is successfully matched and $\bar{A}_{i}(t) = \emptyset$ otherwise. 
Inspired by real applications such as labor market where workers usually update their working experience on their profiles, we also assume each player can observe the successfully matched players $\ch_j(A^{-1}_j(t)) = \bar{A}_{j}^{-1}(t) = \set{p_i: \bar{A}_i(t)=a_j}$ with each arm $a_j \in \cK$ as previous works \cite{liu2021bandit,kong2022thompson,ghosh2022nonstationary,kong2023player,wang2022bandit}.

The matching $\bar{A}(t)$ at round $t$ is the set of all pairs $(p_i,\bar{A}_i(t))$. 
Stability of matchings is a key concept that describes the state in which any player or arm has no incentive to abandon the current partner \citep{gale1962college,roth1992two}. 
Formally, a matching is stable if it cannot be improved by any arm or player-arm pair.
Specifically, an arm $a_j$ improves $\bar{A}(t)$ if $\ch_j( \bar{A}^{-1}_j(t)) \neq \bar{A}_j^{-1}(t)$. That's to say, arm $a_j$ would not accept all members in $\bar{A}^{-1}_j(t)$ when faced with this set. A pair $(p_i,a_j)$ improves the matching $\bar{A}(t)$ if $p_i$ prefers $a_j$ to $\bar{A}_i(t)$ and $a_j$ would accept $p_i$ when faced with $\bar{A}^{-1}_j(t)\cup \set{p_i}$, i.e., $p_i \in \ch_j( \bar{A}^{-1}_j(t)\cup \set{p_i} )$.
That's to say, $p_i$ prefers arm $a_j$ than its current partner and $a_j$ would also accept $p_i$ if $p_i$ apply for $a_j$ together with $a_j$'s current partners
\citep{kelso1982job,abdulkadirouglu2005college,roth1992two}.

Responsive preferences are widely studied in many-to-one markets which guarantee the existence of a stable matching \citep{roth1992two,wang2022bandit}. 
Under this setting, each arm $a_j$ has a preference ranking over individual players and a capacity $C_j>0$. When a set of players propose to $a_j$, it accepts $C_j$ of them with the highest preference ranking. 
This case recovers the one-to-one matching when $C_j=1$. 
For convenience, define $C=\sum_{j\in[K]} C_j$ as the total capacities of all arms. 
Apart from responsiveness, we also consider a more general substitutability setting in Section \ref{sec:decen}.

In this paper, we study the bandit problem in many-to-one matching markets with responsive and substitutable preferences. Under both properties, the set $M^*$ of stable matchings between $\cN$ and $\cK$ is non-empty \citep{roth1992two,kelso1982job}. 
For each player $p_i$, let $\overline{m}_i\in[K]$ and $\underline{m}_i\in[K]$ be the index of $p_i$'s most and least favorite arm among all arms that can be matched with $p_i$ in a stable matching, respectively.
The objective of each player $p_i$ is to minimize the cumulative stable regret defined as the cumulative difference between the reward of the stable arm and that the player receives during the horizon. The player-optimal and pessimal stable regret are defined as
\begin{align}
\overline{R}_i(T) = \EE{\sum\nolimits_{t=1}^T \bracket{\mu_{i,\overline{m}_i} - X_{i,A_i(t)}(t)} } \,,\\
\underline{R}_i(T) = \EE{\sum\nolimits_{t=1}^T \bracket{\mu_{i,\underline{m}_i} - X_{i,A_i(t)}(t)} } \,,
\end{align}
respectively \citep{liu2020competing,zhang2022matching,kong2023player,wang2022bandit}. The expectation is taken over by the randomness in reward gains and the players' policies. 
For convenience, we define the preference gaps to measure the hardness of the problem. 

\begin{definition}\label{def:gap}
    For each player $p_i$ and arm $a_j \neq a_{j'}$, define $\Delta_{i,j,j'} = \abs{\mu_{i,j}-\mu_{i,j'}}$ as the preference gap of $p_i$ between $a_j$ and $a_{j'}$. Let $\rho_i$ be the preference ranking of player $p_i$, where $\rho_{i,k}$ represents the arm ranked $k$-th in $p_i$'s preference. With a little abuse of notation, denote $\rho_i(a_j)$ as the rank of $a_j$ in $p_i$'s preference. Define $\Delta_{\min} = \min_{i,k\in[\min\set{N,K-1}]}\Delta_{i,\rho_{i,k},\rho_{i,k+1}}$ as the minimum preference gap between the arm ranked the first $\min\set{N+1,K}$-th among all players, $\Delta_{\overline{m}} = \min_{i,k\in[\rho_i(\overline{m}_i)]}\Delta_{i,\sigma_{i,k},\sigma_{i,k+1}}$ as the minimum preference gap between the arm ranked the first $(\rho_i(\overline{m}_i)+1)$-th among all players and $\Delta_{\underline{m}} = \min_{i,k>\rho_i(\underline{m}_i)}\Delta_{i,\underline{m}_i,\rho_{i,k}}$ as the minimum preference gap between $\underline{m}_i$ and any arm that has lower ranking than $\underline{m}_i$ among all players.  
\end{definition}

%!TEX root =  main.tex
\section{An Extension of \citet{kong2023player}
}\label{sec:etgs}

Recall that \citet{kong2023player} provide a near-optimal bound $O(K\log T/\Delta_{\min}^2)$ for player-optimal stable regret in one-to-one markets. We first provide an extension of their algorithm, explore-then-deferred-acceptance (ETDA), for many-to-one markets with responsiveness and $N\le K\cdot \min_{j\in[K]} C_j$. 

The deferred acceptance (DA) algorithm is designed to find a stable matching when both sides of participants have known preferences. The algorithm proceeds in multiple steps. At the first step, all players propose to their most preferred arm and each arm rejects all but their favorite subset of players among those who propose to it. Such a process continues until no rejection happens. It has been shown that the final matching is the player-optimal stable matching under responsiveness \citep{gale1962college,kelso1982job,roth1992two}.

Since players are uncertain about their preferences, the ETDA algorithm lets players first explore to learn this knowledge and then follow DA to find a stable matching. 
Specifically, each player first estimates an index in the first $N$ rounds (phase $1$); and then explores its unknown preferences in a round-robin way based on its index (phase $2$). After estimating a good preference ranking, it will follow DA to find the player-optimal stable matching (phase $3$). 
Compared with \citet{kong2023player}, the difference mainly lies in the first phase of estimating indices for players where multiple players can share the same index in many-to-one markets. For completeness, we provide the detailed algorithm in Appendix \ref{sec:etda:appendix} and the theoretical guarantees below.

\begin{theorem}\label{thm:etda}
Under the responsiveness condition, when $N\le K\cdot \min_{j\in[K]}C_j$, the player-optimal stable regret of each player $p_i$ by following ETDA satisfies
   \begin{align}
        \overline{R}_i(T) &\le O\bracket{K\log T/\Delta_{\min}^2}  \,.
    \end{align}
\end{theorem}
Due to the space limit, the proof of Theorem \ref{thm:etda} is deferred to Appendix \ref{sec:proof:etda}. Under the same decentralized setting, this player-optimal stable regret bound is even $O(N^5K\log T/\varepsilon^{N^4})$ better than the weaker player-pessimal stable regret bound in \citet{wang2022bandit}. 
Such a result also achieves the same order as the state-of-the-art analysis in the reduced one-to-one setting \citep{kong2023player}.

Though achieving better regret bound, the ETDA algorithm is not incentive compatible. We can consider the market where the player-optimal stable arm of a player $p_i$ is its least preferred arm. 
If $p_i$ always reports that it does not estimate the preference ranking well, then the stopping condition of phase $2$ is never satisfied. In this case,  all of the other players fail to find a stable matching and suffer $O(T)$ regret, while this player is always matched with more preferred arms than that in the stable matching during phase $2$, resulting in $O(T)$ improvement in the cumulative rewards.  Thus player $p_i$ lacks the incentive to always act as the algorithm requires. 
To improve the algorithm in terms of incentive compatibility, we further propose a novel algorithm in the next section. 

%!TEX root =  main.tex

\section{Adaptively ETDA (AETDA) Algorithm}\label{sec:aetda}

In this section, we propose a new algorithm adaptively ETDA (AETDA) for many-to-one markets with responsive preferences which is incentive compatible. 
To ensure each player has a chance to be matched, we simply assume $N\le C$ as existing works in many-to-one and one-to-one markets \citep{liu2020competing,liu2021bandit,zhang2022matching,kong2023player,wang2022bandit}, which relaxes the requirement of ETDA in the previous section.

For simplicity, we present the main algorithm in a centralized manner in Algorithm \ref{alg:AETDA}, i.e., a central platform coordinates players' selections in each round. The discussion on how to extend it to a decentralized setting is provided later.

\begin{algorithm}[thb!]
    \caption{centralized adaptively explore-then-deferred-acceptance (AETDA, from the view of the central platform)}\label{alg:AETDA}
    \begin{algorithmic}[1]
    \STATE Initialize: $S_i = \cK, E_i = \true$ for each player $p_i \in \cN$ \label{alg:AETDA:initial}
    \FOR{round $t=1,2,...,$}
        \STATE Allocate $A_i(t) \in S_i$ to each player $p_i$ with $E_i = \true$ in a round-robin manner; Allocate $A_i(t) = \mathrm{opt}_i$ to each player $p_i$ with $E_i = \false$ \label{alg:AETDA:exploit}
        \STATE Receive the estimation status $\mathrm{opt}_i$ from each $p_i$ \label{alg:AETDA:receiveOpt}
        \FOR{each player $p_i\in \cN$ with $\mathrm{opt}_i \neq -1$}
                \STATE $E_i = \false$ \label{alg:AETDA:opt:true:update}
        \ENDFOR
        \FOR{each player $p_i\in \cN$ and $a_j \in S_i$ with $p_i \notin \ch_j(\set{p_{i'}: \mathrm{opt}_{i'}=a_j }\cup \set{p_i})$} \label{alg:aetda:detect:start} 
                \STATE $S_i = S_i \setminus \set{a_j}$ \label{alg:AETDA:update:Si}
                \STATE Set $E_i = \true$ if $E_i=\false$ and $a_j = \mathrm{opt}_i$ \label{alg:aetda:detect:E:false}
    \ENDFOR\label{alg:aetda:detect:end} 
    \ENDFOR
    \end{algorithmic}
\end{algorithm}

Intuitively, AETDA integrates the learning process into each step of DA instead of estimating the full preference ranking well before running DA. More specifically, each player explores arms in a round-robin manner in each step to learn its most preferred arm and then focuses on this arm before being rejected in the corresponding step of DA. 
For each player $p_i$, the algorithm maintains $S_i$ to represent the available arm set that has not rejected $p_i$ in previous steps
and $E_i$ to represent the exploration status. Specifically, $E_i=\true$ means that $p_i$ still needs to explore arms in a round-robin manner to find its most preferred arm in $S_i$, and $E_i=\false$ means that $p_i$ now focuses on its most preferred available arm. At the beginning of the algorithm, $S_i$ is initialized as the full arm set $\cK$ and $E_i$ is initialized as $\true$ (Line \ref{alg:AETDA:initial}).

For players with $E_i = \true$, the central platform would allocate the arm $A_i(t)\in S_i$ in a round-robin manner.
And for those players with $E_i=\false$, they can just focus on the determined optimal arm $\mathrm{opt}_i$ (Line \ref{alg:AETDA:exploit}).
After being matched in each round, each player $p_i$ would update its empirical mean $\hat{\mu}_{i,A_i(t)}$ and the number of observed times $T_{i,A_i(t)}$ on arm $A_i(t)$ as $\hat{\mu}_{i,A_i(t)} = ({\hat{\mu}_{i,A_i(t)}\cdot T_{i,A_i(t)} + X_{i,A_i(t)}(t) })/{(T_{i,A_i(t)}+1)}\,,\ T_{i,A_i(t)} = T_{i,A_i(t)}+1$. 
For the preference value $\mu_{i,j}$ towards each arm $a_j$, $p_i$ also maintains a confidence interval at $t$ with the upper bound $\ucb_{i,j}
:= \hat{\mu}_{i,j}+\sqrt{{6\log T}/{T_{i,j}}}$ 
and lower bound $\lcb_{i,j}
:=\hat{\mu}_{i,j}-\sqrt{{6\log T}/{T_{i,j}}}$.  
If $T_{i,j}=0$, $\ucb_{i,j}$ and $\lcb_{i,j}$ are set as $\infty$ and $-\infty$, respectively. 
When the $\ucb$ of $a_j$ is even lower than the $\lcb$ of other available arms, $a_j$ is considered to be less preferred. Based on the estimations, $p_i$ needs to determine whether an arm can be considered as optimal in $S_i$ and submit this status to the platform (Line \ref{alg:AETDA:receiveOpt}). Specifically, if there exists an arm $a_j \in S_i$ such that $\lcb_{i,j} > \max_{a_{j'}\in S_i\setminus \set{a_j}}\ucb_{i,j'}$, then $a_j$ is regarded as optimal and player $p_i$ would submit $\mathrm{opt}_i = a_j$ to the platform. Otherwise, no arm can be regarded as optimal, and $p_i$ would submit $\mathrm{opt}_i=-1$. 
For players who have learned their most preferred arm, the platform would mark their exploration status as $\false$ (Line \ref{alg:AETDA:opt:true:update}). 

To avoid conflict when players with $E_i=\true$ explore arms in a round-robin manner, we introduce a detection procedure to detect whether an arm in $S_i$ is occupied by its more preferred players (Line \ref{alg:aetda:detect:start}-\ref{alg:aetda:detect:end}).   
Specifically, if an arm $a_j$ does not accept player $p_i$ when faced with the player set who regards $a_j$ as the optimal one (Line \ref{alg:aetda:detect:start}), then $p_i$ can be regarded to be rejected by $a_j$ when exploring this arm. 
In this case, no matter whether this arm is the most preferred one, $p_i$ has no chance of being matched with it. 
So $p_i$ directly deletes $a_j$ from its available arm set $S_i$ (Line \ref{alg:AETDA:update:Si}). And if this arm is just the estimated optimal arm of $p_i$, then this case is equivalent in offline DA to that $p_i$ is rejected when proposing to its most preferred arm. 
% (Line \ref{alg:aetda:detect:E:false:start}). 
In this case, $p_i$ needs to explore to learn its next preferred arm and update $E_i$ as $\true$ (Line \ref{alg:aetda:detect:E:false}).

For the arrangement of round-robin exploration, without loss of generality, we can convert the original set of $K$ arms with total capacity $C$ into a set of $C$ new arms, each with a capacity $1$. When $N$ players explore these $C$ new arms: the platform let $p_1$ follow the ordering $1,2,...,C-1,C,1,...$; $p_2$ follow $2,3,...,C,1,2,...$; and so on.
If an arm $a_j$ is unavailable for a player $p_i$, $p_i$ simply forgo the opportunity to select in the corresponding rounds. This pre-arranged ordering ensures that, in the worst case, each player can match with each available new arm, and so as to the available original arm, at least once in every $C$ rounds.

\paragraph{Extension to the decentralized setting. }
In the decentralized setting without a central platform, each player maintains and updates their own $S_i$ and $E_i$. We can define a phase version of Algorithm \ref{alg:AETDA}. Specifically, each phase contains a number of rounds and the size of phases grows exponentially, i.e., $2,2^2,2^3,\cdots$. Within each phase, each player $p_i$ would explore arms in $S_i$ in a round-robin manner if $E_i=\true$ as discussed above and focus on arm $\mathrm{opt}_i$ otherwise. 
Players only update the status of $\mathrm{opt}_i$ (Line \ref{alg:AETDA:receiveOpt}), $E_i$ (Line \ref{alg:AETDA:opt:true:update}),  and $S_i$ (Line \ref{alg:aetda:detect:start}-\ref{alg:aetda:detect:end}) at the end of the phase based on the communication with other players and arms. 
If $L$ observations on arms are enough to learn the optimal one in the centralized version, then the stopping condition (Line \ref{alg:AETDA:receiveOpt}) would be satisfied at the end of the phase guaranteeing the number of observations in this decentralized version and the total number of selecting times would be at most $2L$ due to the exponentially increasing phase length. 
So the regret in this decentralized version is at most two times as that suffered in the centralized version. And the number of communications is at most $O(\log T)$ which is of the same order as the ETDA algorithm and also \citet{kong2023player} for the one-to-one setting.

\subsection{Theoretical Analysis}

Algorithm \ref{alg:AETDA} presents a new perspective that integrates the learning process into each step of the DA algorithm to find a player-optimal stable matching. In the following, we will show that such a design simultaneously enjoys guarantees of player-optimal stable regret and incentive compatibility.

\begin{theorem}\label{thm:aetgs:regret}
Under the responsiveness condition, when $N\le C$, the player-optimal stable regret of each player $p_i$ by following Algorithm \ref{alg:AETDA} satisfies
    \begin{align*}
        \overline{R}_i(T) \le O\bracket{ N\cdot \min\set{N, K}C \log T/\Delta_{\overline{m}}^2  } \,.
    \end{align*}
\end{theorem}

% The following theorem further discusses the incentive compatibility of Algorithm \ref{alg:AETDA}. 

\begin{theorem}{(Incentive  Compatibility)}\label{thm:aetgs:strategic}
When all of the other players follow Algorithm \ref{alg:AETDA}, no single player $p_i$ can improve its final matched arm by misreporting $\mathrm{opt}_i$ in some rounds.  
\end{theorem}

Compared with \citet{wang2022bandit}, our result not only achieves 
an $O(N^4K\log T/(C\varepsilon^{N^4}))$ improvement over their weaker player-pessimal stable regret objective but also enjoys guarantees of incentive compatibility. 
Compared with the state-of-the-art result in one-to-one settings, our algorithm is more robust to players' deviation only with the cost of $O(NC)$ worse regret bound \citep{zhang2022matching,kong2023player}. 
To the best of our knowledge, it is the first algorithm that simultaneously achieves guarantees of polynomial player-optimal stable regret and incentive compatibility in both many-to-one markets and previously widely studied one-to-one markets without knowing the value of $\Delta$. 

Due to the space limit, the proofs of two theorems are deferred to Appendix\ref{sec:proof:aetda}.

\section{Online DA Algorithm for Substitutability }\label{sec:decen}
% Preferences

In many-to-one markets, arms may have combinatorial preferences over groups of players, which may not be well characterized by responsiveness. In this setting, we consider the markets with substitutability, which is one of the most common and general conditions that ensure the existence of a stable matching and is defined below.

\begin{definition}{(Substitutability)}\label{def:substi}
The preference of arm $a_j$ satisfy substitutability if for any player set $P\subseteq \cN$ that contains $p_i$ and $p_{i'}$, $p_i \in \ch_j(P\setminus \set{p_{i'}})$ when $p_i \in \ch_j(P)$.
\end{definition}

The above property states that arm $a_j$ keeps accepting player $p_i$ when other players become unavailable. This is the sense that $a_j$ regards players in a team as substitutes rather than complementary individuals (in which case the arm may give up accepting the player when others become unavailable). 
Such a phenomenon appears in many real applications and covers responsiveness as proved below.

\begin{remark} 
Select a player set $P\subseteq \cN$ which contains $p_{i}$ and $p_{i'}$. 
Suppose $p_i \in \ch_j(P)$, i.e., $p_i$ is one of the $C_j$ highest-ranked players in $P$.  
Then when the available set becomes $P\setminus\set{p_{i'}}$, 
$p_i$ is still one of the $C_j$ highest-ranked players, i.e., $p_i \in \ch_j(P\setminus\set{p_{i'}})$.
\end{remark}

The substitutability property is more general than responsiveness as arms' preferences can have combinatorial structures. The following is an example that satisfies substitutability but not responsiveness \cite{roth1992two}. 

\begin{example}
$\cN=\set{p_1,p_2,p_3}$ and $ \cK=\set{a_1, a_2}$. 
The arms' preference rankings over subsets of players are
\begin{itemize}
    \item $a_1: \set{p_1,p_2},\set{p_1,p_3},\set{p_2,p_3},\set{p_3},\set{p_2},\set{p_1}$.
    \item $a_2: \set{p_3},\emptyset$.
\end{itemize}
That is to say, $\ch_j(P)$ is the subset that ranks highest among all subsets listed above that only contain players in $P$. Taking the preferences of $a_2$ as an example, when $p_3 \in P$, then $\ch_j(P)=\set{p_3}$; Otherwise, $\ch_j(P)=\emptyset$. 
\end{example}

For many-to-one markets with substitutable preferences, we propose an online deferred acceptance (ODA) algorithm (presented in Algorithm \ref{alg:decen}). ODA is inspired by the idea of the DA algorithm with the arm side proposing, which finds a player-pessimal stable matching when players know their preferences. 
Specifically, the DA algorithm with the arm proposing proceeds in several steps. 
In the first step, each arm proposes to its most preferred subset among all players. 
Each player would reject all but the most preferred arm among those who propose it. 
In the following each step, each arm still proposes to its most preferred subset of players among those who have not rejected it and each player rejects all but the most preferred one among those who propose to it. This process stops when no rejection happens and the final matching is the player-pessimal stable matching \citep{kelso1982job,roth1992two}.

\begin{algorithm}[thb!]
    \caption{online deferred acceptance (from view of $p_i$)}\label{alg:decen}
    \begin{algorithmic}[1]
    \STATE Input: player set $\cN$, arm set $\cK$ \label{alg:decen:input}\\
        \STATE Initialize: $P_{i,j}=\cN,  \hat{\mu}_{i,j}=0, T_{i,j}=0$ for each $j\in [K]$; $S_i(1)=\set{a_j\in \cK: p_i \in \ch_j(P_{i,j})}$ \label{alg:decen:initial} 
        \FOR{each round $t=1,2,\cdots$}
            \STATE Select $A_i(t)\in S_i(t)$ in a round-robin way \label{alg:decen:roundrobin}
            \STATE Update $\hat{\mu}_{i,\bar{A}_i(t)}$ and $T_{i,\bar{A}_i(t)}$ if $\bar{A}_i(t) = A_i(t) \neq \emptyset$
        \label{alg:decen:updatemuT}
            \STATE $S_i(t+1)=S_i(t)$
            \FOR{$a_j \in S_{i}(t)$ and $\ucb_{i,j}(t) < \max_{a_{j'} \in S_{i}(t)} \lcb_{i,j'}(t)$}\label{alg:decen:delete:suboptimal:condition}
                	\STATE $S_i(t+1) = S_i(t+1)\setminus \set{a_j}$ \label{alg:decen:delete:suboptimal}
            \ENDFOR
                \IF{$t\ge 2$ and $\forall p_{i'}\in \cN: \bar{A}_{i'}(t)=\bar{A}_{i'}(t-1)$}\label{alg:decen:update:available:start}
                    \STATE $\forall j\in[K]$, $P_{i,j} = P_{i,j}\setminus\set{p_{i'}: \bar{A}_{i'}(t)\neq j, \exists t'<t-1 \text{ s.t. } \bar{A}_{i'}(t')=j }$\label{alg:decen:update:available}
                    \STATE $S_{i}(t+1) = \set{a_j: p_i \in \ch_j(P_{i,j}) }$ \label{alg:decen:update:plausible}
                \ENDIF \label{alg:decen:update:available:end}
    \ENDFOR
    \end{algorithmic}
\end{algorithm}

The ODA algorithm is designed with the guidance of this procedure but players decide which arm to select in each round. 
Specifically, each player $p_i$ needs to record the available player set $P_{i,j}$ for each arm $a_j$, which consists of players who have not rejected arm $a_j$ and is initialized as the full player set $\cN$. 
Then if a player $p_i$ is in the choice set of $a_j$ when the set $P_{i,j}$ of players is available, i.e., $p_{i} \in \ch_j(P_{i,j})$,  $p_i$ would be accepted if it proposes to $a_j$ together with other players in $P_{i,j}$. 
The main purpose of the algorithm is to let players wait for this opportunity to choose arms that will successfully accept them. 

% , rather than blindly proposing to arms at the beginning and dealing with the frequent conflicts.

Each player $p_i$ can further construct the plausible set $S_i$ to contain those arms that may successfully accept it, i.e., $S_i = \set{ a_j: p_i \in \ch_j(P_{i,j})}$.  
Here for simplicity, we additionally assume each player $p_i$ knows whether $p_i \in \ch_j(P)$ for the possible $P$. This assumption is only used for clean analysis and later we show that players can learn this knowledge during the algorithmic operation with an additional constant number of rounds. 
% the algorithm can also be generalized to the case where this information is unavailable by letting players in $P_{i,j}$ pull $a_j$ and observe whether it is accepted. Since arms know their own preferences and conflicts are deterministically resolved, at most $2^N$ rounds are needed to obtain this information.
Apart from $P_{i,j}$ and $S_i$, each player $p_i$ also maintains $\hat{\mu}_{i,j}$ and $T_{i,j}$ to record the estimated value for $\mu_{i,j}$ and the number of its observations. At the beginning, both values are initialized to $0$.

In each round $t$, each player $p_i$ proposes to arm $a_j$ in the plausible set $S_i(t)$ in a round-robin way (Line \ref{alg:decen:roundrobin}). If they are successfully matched with each other (Line \ref{alg:decen:updatemuT}), $p_i$ would update the corresponding $\hat{\mu}_{i,j}, T_{i,j}$ as the previous section. 
When the $\ucb$ of $a_j$ is even lower than the $\lcb$ of other plausible arms, $a_j$ is considered to be less preferred. 
In this case, the final stable arm of player $p_i$ must be more preferred than $a_j$ and thus there is no need to select $a_j$ anymore (Line \ref{alg:decen:delete:suboptimal}).

Recall that the plausible sets of players are constructed based on the available sets for arms. 
To ensure each player successfully be accepted by arms in their own plausible set, all players need to keep the available sets for arms updated in sync. 
With the awareness that players always select plausible arms in a round-robin way, once $p_i$ observes that all players focus on the same arm in the recent two rounds, it believes all players have determined the most preferred one. 
In this way, $p_i$ updates the available set $P_{i,j}$ for each arm $a_j$ by deleting players who do not consider $a_j$ as stable arms (Line \ref{alg:decen:update:available}).
Since all players have the same observations, the update times of $P_{i,j}$ would be the same. 
Such a stage in which all players determine the most preferred arm in the plausible set can just be regarded as a step of the offline DA algorithm (with the arm side proposing) where each player rejects all but the most preferred one among those who propose to it. 
Thus, the update times of $P_{i,j}$ just divide the total horizon into several stages, corresponding to a step of DA. 

We then introduce how to learn $\ch_j(P)$ for the possible player set $P$ in the algorithm. At the beginning, we introduce an initialization phase of $K$ rounds. Each of the $K$ rounds corresponds to one arm. And in the round for arm $a_j$, all players in $P_{i,j}=\cN$ select arm $a_j$ and others select nothing. Players can then learn whether $p_i\in \ch_j(P_{i,j})$ based on whether it is accepted.
% All players select each arm once and learn whether $p_i\in \ch_j(\cN)$ for each player $p_i$, arm $a_j$ based on whether the player is accepted by the corresponding arm. 
During the algorithmic operation, when all players identify their most preferred arms and update $P_{i,j}$ for all arm $a_j$ (Line \ref{alg:decen:update:available}), we also additionally introduce $K$ rounds after this round. Players perform the same operation as the initialization but with the updated $P_{i,j}$. 
% And in the round for arm $a_j$, all players in $P_{i,j}$ select arm $a_j$ and others select nothing. 
Since the algorithm proceeds in at most $NK$ stages (each player rejects each arm at most once) and each stage requires $K$ rounds to learn the arms' preference knowledge, this totally introduces additional $NK^2$ rounds, which is a constant term and does not influence the main regret order. 
 
\subsection{Theoretical Analysis}\label{sec:main:analysis}

% In this section, we provide the regret bound and discuss the incentive compatibility of Algorithm \ref{alg:decen}.
We first provide the regret guarantee for Algorithm \ref{alg:decen}.

\begin{theorem}\label{thm:decen}
Under the substitutability condition, when players know arms' exact preferences, the player-pessimal stable regret of each player $p_i$ by following Algorithm \ref{alg:decen} satisfies
    \begin{align}
        \underline{R}_i(T) \le 
        % NK \bracket{\frac{128\log T}{\Delta^2}+2+ \frac{\pi^2}{3}} \mu_{i,\underline{m}_i}  =
        O(NK\log T/\Delta_{\underline{m}}^2) \,.
    \end{align}
\end{theorem}

Due to the space limit, the proof is provided in Appendix\ref{sec:proof:decen}.  
To the best of our knowledge, Theorem \ref{thm:decen} is the first theoretical result for bandit learning in many-to-one matching markets with combinatorial substitutable preferences. 
Our algorithm not only works in more general markets but also achieves a significant improvement from $O(N^5K^2 \log^2 T/(\varepsilon^{N^4}\Delta^2))$ to $O(NK\log T/\Delta^2)$ in the recovered responsiveness setting \citep{wang2022bandit}.

% \paragraph{Known arms' preferences. }

% Note that in the recovered many-to-one markets with responsive preferences, to get the knowledge of whether $p_i\in \ch_j(P_{i,j})$ for every player $p_i$, arm $a_j$ and possible set $P_{i,j}\subseteq \cN$, only additional $N^2K$ rounds are needed by letting every two players propose to each arm to check who is more preferred, which is only a constant term.
% This knowledge is also used in previous decentralized many-to-one markets \citep{wang2022bandit}. 
% Under the exact same assumptions of knowing arms' preferences and observing their matched players as \citet{wang2022bandit}, our algorithm achieves a significant improvement .

% we provide a proof sketch of Theorem \ref{thm:decen} in Appendix \ref{sec:proof:sketch} and the detailed proof is deferred to Appendix \ref{sec:proof:full}. 

% \input{experiment}

\section{Conclusion}\label{sec:conclusion}

In this paper, we study the bandit learning problem in many-to-one markets. 
We first extend the result of \citet{kong2023player} to the many-to-one markets with responsive preferences and provide a player-optimal regret bound. Since such an algorithm lacks incentive compatibility, we further propose the AETDA algorithm which enjoys a guarantee of player-optimal regret and is simultaneously incentive compatible. 
We also consider a more general setting with substitutable preferences and provide an upper bound for player-pessimal stable regret. 
Compared with existing works for many-to-one markets \citep{wang2022bandit}, our algorithms achieve a significant improvement in terms of not only regret bound but also guarantees of incentive compatibility.

An interesting future direction is to optimize the player-optimal stable regret in the general many-to-one markets with substitutable preferences. All of the previous algorithms for the reduced settings go through based on the uniform exploration strategy. However, under substitutability, an arm may accept none of the candidates which makes it challenging for players to perform such a strategy.

% \newpage

% \section*{Acknowledgements}
% The corresponding author Shuai Li is supported by National Key Research and Development Program of China (2022ZD0114804) and National Natural Science Foundation of China (62376154, 62076161).

% \bibliographystyle{aaai24.bst}
\bibliography{ref}

\appendix
\onecolumn

\section{The ETDA Algorithm}
\label{sec:etda:appendix}

\subsection{Algorithmic Description}

Inspired by \citet{kong2023player}, we propose a more efficient explore-then-DA (ETDA) algorithm for many-to-one markets. 
Recall that each arm $a_j$ has a capacity $C_j$ under responsiveness. Denote $C_{\min} = \min_{j \in [K]} C_j$ as the minimum capacity among all arms and $j_{\min} \in \argmin_{j \in [K]} C_j$ as one arm that has the minimum capacity. 
The following algorithm runs when $N \le K\cdot C_{\min}$.

Following ETDA, each player would first estimate an index in the first $N$ rounds (Line \ref{alg:etda:index}); then explore its unknown preferences in a round-robin way based on the estimated index (Line \ref{alg:p2:start}-\ref{alg:p2:end}). After estimating a good preference ranking, it will follow DA with the player side proposing to find a player-optimal stable matching (Line \ref{alg:p3:start}-\ref{alg:p3:end}).

\begin{algorithm}[thb!]
    \caption{explore-then-DA (ETDA, from view of $p_i$)}\label{alg:etda}
    \begin{algorithmic}[1]
    \STATE Input: player set $\cN$, arm set $\cK$ \label{alg:input}
    \STATE Initialize: $\hat{\mu}_{i,j}=0, T_{i,j}=0, \forall j\in[K]; \mathrm{Flag}=\false$
    \STATE For $t\in[N]$: estimate an index $\mathrm{Index}$ \label{alg:etda:index}
    \FOR{$\ell=1,2,\ldots $}\label{alg:p2:start}
        \FOR{$t=N+2^{\ell}-1,\ldots, N+2^{\ell}-1+2^{\ell}$} \label{alg:p2:learn:start} 
            \STATE  $A_i(t)=a_{(\mathrm{Index}+t-1)\%K+1}$ \label{alg:p2:learn:select}
            \STATE Observe $X_{i,A_i(t)}(t)$ and update $\hat{\mu}_{i,A_i(t)}$, $T_{i,A_i(t)}$ if $\bar{A}_i(t) = A_i(t) $ \label{alg:p2:learn:update}
        \ENDFOR\label{alg:p2:learn:end}
        \STATE $t = N+2^{\ell}+2^{\ell}$\label{alg:monitor:round}
        \STATE Compute $\ucb_{i,j}$ and $\lcb_{i,j}$ for each $j\in[K]$ \label{alg:p2:computeUCBLCB}
        \IF{$N< K$}\label{alg:estimate:ranking:start}
            \IF{$\exists \sigma$ such that $\lcb_{i,\sigma_{k}}>\ucb_{i,\sigma_{k+1}}$ for any $k\in[N]$ and $\lcb_{i,\sigma_{N}}>\ucb_{i,\sigma_{k}}$ for any $k=N+2,\ldots, K$ } 
            \STATE $\mathrm{Flag}=\true$
            \ENDIF
        \ELSE
            \IF{$\exists \sigma$ such that $\lcb_{i,\sigma_{k}}>\ucb_{i,\sigma_{k+1}}$ for any $k\in[K-1]$} 
            \STATE $\mathrm{Flag}=\true$
            \ENDIF
        \ENDIF\label{alg:estimate:ranking:end}
        \IF{$\mathrm{Flag}=\true$}\label{alg:p2:rank:start}
        \STATE $A_i(t) = a_{\mathrm{Index}}$
        \STATE Enter DA phase with $\sigma$ if $\cup_{j\in[K]} \set{\bar{A}^{-1}_{j}(t)}= \cN$ \label{alg:p2:enterP3}
        \ELSE 
        \STATE $A_i(t) = \emptyset$
        \ENDIF \label{alg:p2:rank:end}
    \ENDFOR\label{alg:p2:end} 
    \STATE //DA phase: initialize $s=1$ \label{alg:p3:start}
    \STATE Always propose $a_{\sigma_{s}}$; update $s=s+1$ if $p_i$ is rejected \label{alg:p3:end}
    \end{algorithmic}
\end{algorithm}

At the $1$st round, all players propose to arm $a_{j_{\min}}$. And $a_{j_{\min}}$ would accept $C_{\min}$ of them. Those accepted players get an index $1$. At the following each round $t$, players who are rejected in all of the previous rounds would still propose to $a_{j_{\min}}$ and other players would propose to any other arm except for $a_{j_{\min}}$. Among those who propose to $a_{j_{\min}}$,  $C_{\min}$ of them would then be accepted and get index $t$. Following this process, all players would get an index at the end of $N$th round as $C_{\min}>0$. 
 % (arms with $0$ capacity can be removed from the market)

Since only no more than $C_{\min}$ players have the same index, players sharing the same index can be successfully accepted when they propose to any arm. Thus all players can explore arms in a round-robin way based on their indices.  
The exploration phase is broken into several epochs: the $\ell$th epoch contains an exploration block of length $2^\ell$ and a communication round. During the exploration block (Line \ref{alg:p2:learn:start}-\ref{alg:p2:learn:end}), players would propose to arms according to their indices in a round-robin way. And at the communication round, players try to estimate all players' estimation status in the market. 
For this purpose, each player needs to first determine its own estimation status. Specifically, each player $p_i$ would first compute a confidence interval for each $\mu_{i,j}$ with UCB and LCB to be the upper and lower bound. If the confidence intervals towards two arms are disjoint, the player can determine its preferences over these two arms. 
So once the player identifies the ranking of the first $\min\set{N,K}$ most preferred arms based on the estimated preferences, it can determine that its preferences are estimated well (Line \ref{alg:estimate:ranking:start}-\ref{alg:estimate:ranking:end}). 
% \fang{modify}
% If the confidence intervals of the first $\min\set{N+1,K}$-ranked arms are disjoint, the player can determine that its preference ranking towards the most preferred $\min\set{N,K}$ arms has been estimated well and establish the estimated ranking $\sigma$ based on the estimated preference values (Line \ref{alg:p2:rank:start}-\ref{alg:p2:rank:end}). 
Players can also transmit their current estimation status to others through its action: if a player estimates its preferences well, it will propose to the arm labeled by its index; otherwise, it will give up the proposing chance in this round (Line \ref{alg:p2:rank:start}-\ref{alg:p2:rank:end}). 
Recall that all players would be accepted when proposing to the arm together with other players having the same index. Thus if a player observes that all players are successfully matched in this round, it can determine all players have estimated their unknown preferences well and would enter the DA phase to find a stable matching (Line \ref{alg:p2:enterP3}).

In the DA phase, all players would act based on the procedure of the offline DA algorithm with the player side proposing \cite{roth1984stability,roth1992two}. At the first round of the DA phase, all players propose to their most preferred arm according to their estimated rankings. And each arm $a_j$ would only accept the top $C_j$ highest players among those who propose it. 
In the following each round, each player still proposes to its most preferred arm among those who have not rejected it, and each arm accepts its most preferred $C_j$ players among those who propose to it.  
Until no rejection happens, all players would not change their actions in the following rounds. 
% Since each arm can reject each player at most once, such a process would continue for at most $NK$ rounds before converging. 
According to Lemma \ref{lem:DA:steps}, if the estimated preference ranking for the most preferred $\min\set{N,K}$ arms of each player is correct, this process is equivalent to the offline DA algorithm with the player side proposing and the final matching is shown to be player-optimal \citep{roth1984stability,roth1992two}.

\subsection{Proof of Theorem \ref{thm:etda}}\label{sec:proof:etda}

Before the main proof, we first introduce some notations that will be used in the full Appendix. 
Let $T_{i,j}(t),\hat{\mu}_{i,j}(t)$ be the value of $T_{i,j},\hat{\mu}_{i,j}$ at the end of round $t$. Define the bad event 
$\cF = \set{\exists t\in[T], i\in[N],j\in[K], \abs{\hat{\mu}_{i,j}(t)-\mu_{i,j}} > \sqrt{\frac{6\log T}{ T_{i,j}(t)}} }$
to represent that some estimations are far from the real preference value at some round.

% Define the bad event 
% $\cF = \set{\exists t\in[T], i\in[N],j\in[K], \abs{\hat{\mu}_{i,j}(t)-\mu_{i,j}} > \sqrt{\frac{6\log T}{ T_{i,j}(t)}} }$
% to represent that some estimations are inaccurate at some rounds. 

The player-optimal stable regret of each player $p_i$ by following our ETDA algorithm (Algorithm \ref{alg:etda}) satisfies
\begin{align}
    \overline{R}_i(T) &=\EE{\sum_{t=1}^T \bracket{\mu_{i,\overline{m}_i} - X_{i}(t)} } \notag \\
    &\le \EE{\sum_{t=1}^T \bOne{\bar{A}(t) \neq \overline{m}} \cdot \mu_{i,\overline{m}_i}} \notag \\
    &\le N \mu_{i,\overline{m}_i}+ \EE{\sum_{t=N+1}^T \bOne{\bar{A}(t) \neq \overline{m}} \mid  \urcorner \cF }\cdot \mu_{i,\overline{m}_i} + T\PP{\cF }\cdot \mu_{i,\overline{m}_i} \notag \\
    &\le N \mu_{i,\overline{m}_i}+ \EE{\sum_{t=N+1}^T \bOne{\bar{A}(t) \neq \overline{m}}\mid \urcorner \cF }\cdot \mu_{i,\overline{m}_i}  + 2NK\mu_{i,\overline{m}_i} \label{eq:dueto:bad} \\
    &\le N \mu_{i,\overline{m}_i}+\EE{ \sum_{\ell=1}^{\ell_{\max}  }\bracket{2^{\ell} +1} +\min\set{N^2, NK} }\cdot \mu_{i,\overline{m}_i} + 2NK\mu_{i,\overline{m}_i} \label{eq:dueto:phases:and:def_lmax}\\
    &\le {N \mu_{i,\overline{m}_i}}+  {\bracket{ \frac{192K\log T}{\Delta_{\min}^2} + \log\bracket{ \frac{192K\log T}{\Delta_{\min}^2}} }\cdot \mu_{i,\overline{m}_i}} +{\min\set{N^2, NK}\mu_{i,\overline{m}_i}} + {2NK\mu_{i,\overline{m}_i}} \label{eq:end} \\
    &= O\bracket{K\log T/\Delta_{\min}^2} \notag \,,
\end{align}
where Eq.\eqref{eq:dueto:bad} comes from Lemma \ref{lem:cen:badevent}, Eq. \eqref{eq:dueto:phases:and:def_lmax} holds according to Algorithm \ref{alg:etda} and Lemma \ref{lem:phase3}, Eq. \eqref{eq:end} holds based on Lemma \ref{lem:phase2}.~\\ 
\newline

\begin{lemma}\label{lem:cen:badevent}
\begin{align*}
T\cdot \PP{\cF}  \le 2NK \,.
\end{align*}
\end{lemma}

\begin{proof}
\begin{align}
T\cdot \PP{\cF} &= T\cdot \PP{ \exists t\in[T], i\in[N],j\in[K]: \abs{\hat{\mu}_{i,j}(t)-\mu_{i,j}}>\sqrt{\frac{6\log T}{T_{i,j}(t)}}  } \notag \\
&\le T\cdot \sum_{t=1}^T \sum_{i\in[N]}\sum_{j\in[K]} \PP{ \abs{\hat{\mu}_{i,j}(t)-\mu_{i,j}}>\sqrt{\frac{6\log T}{T_{i,j}(t)}} } \notag\\
&\le T\cdot \sum_{t=1}^T \sum_{i\in[N]}\sum_{j\in[K]} \sum_{w = 1}^{t} \PP{T_{i,j}(t)=w, \abs{\hat{\mu}_{i,j}(t)-\mu_{i,j}}>\sqrt{\frac{6\log T}{T_{i,j}(t)}} } \notag\\ 
&\le T\cdot \sum_{t=1}^T \sum_{i\in[N]}\sum_{j\in[K]} t\cdot 2\exp\bracket{ -3\log T } \label{eq:upper:chernoff}\\ 
&\le 2NK\,. \notag 
\end{align}
where Eq.\eqref{eq:upper:chernoff} comes from Lemma \ref{lem:chernoff}. 
\end{proof}

\begin{lemma}\label{lem:phase3}
Conditional on $ \urcorner \cF$, 
at most $\min\set{N^2, NK}$ rounds are needed in phase 3 before $\sigma_{i,s} = \overline{m}_i$. In all of the following rounds, $s$ would not be updated and $p_i$ would always be successfully accepted by $\overline{m}_i$.
\end{lemma}

\begin{proof}
According to Lemma \ref{lem:cen:ucblcb} and Algorithm \ref{alg:etda}, when player $p_i$ enters the DA phase with $\sigma_i$, 
it holds that the first $\min\set{N,K}$ arms in $\sigma_i$ are the first $\min\set{N,K}$ arms in the real preference ranking of player $p_i$. 
Further, according to Lemma \ref{lem:phase2}, all players enter in the DA phase simultaneously. 
Above all, the procedure of the DA phase is equivalent to the procedure of the offline DA algorithm with the player proposing \citep{roth1984stability} as well as the players' real preference rankings  (Lemma \ref{lem:DA:steps}). 
Thus at most $\min\set{N^2, NK}$ rounds are needed before each player $p_i$ successfully finds the optimal stable arm $\overline{m}_i$. 
Once the optimal stable matching is reached, no rejection happens anymore and $s$ will not be updated. Thus each player $p_i$ would always be accepted by $\overline{m}_{i}$ in the following rounds. 
\end{proof}

\begin{lemma}\label{lem:phase2}
Conditional on $\urcorner \cF$, phase 2 will proceed in at most $\ell_{\max}$ epochs where 
\begin{align}
    \ell_{\max} = \min\set{\ell: \sum_{\ell'=1}^{\ell} 2^{\ell'} \ge 96K\log T/\Delta_{\min}^2}\,, \label{eq:def:ellmax}
\end{align}
which implies that $\sum_{\ell'=1}^{\ell_{\max}} 2^{\ell'} \le 192K\log T/\Delta_{\min}^2$ and $\ell_{\max} = \log \bracket{\log \bracket{192K\log T/\Delta_{\min}^2} }$  since the epoch length grows exponentially. 
And all players will enter in the DA phase simultaneously at the end of the $\ell_{\max}$-th epoch. 
\end{lemma}

\begin{proof}
Since players propose to arms based on their distinct indices in a round-robin way and $C_j \ge C_{\min},\forall j\in[K]$, all players can be successfully accepted at each round during the exploration rounds. Thus at the end of the epoch $\ell_{\max}$ defined in Eq. \eqref{eq:def:ellmax}, it holds that $T_{i,j}\ge 96\log T/\Delta_{\min}^2$ for any $i\in[N],j\in[K]$.  
 
According to Lemma \ref{lem:cen:pulltime}, when $T_{i,j}\ge 96\log T/\Delta_{\min}^2$ for any arm $a_j$, player $p_i$ finds a permutation $\sigma_i$ over arms such that $\lcb_{i,\sigma_{i,k}}>\ucb_{i,\sigma_{i,k+1}}$ for any $k\in[\min\set{N,K-1}]$ and $\lcb_{i,\sigma_{i,N}}>\ucb_{i,\sigma_{i,k}}$ for any $k=N+2,\ldots,K$ if $N<K$.

Thus, at the communication round of epoch $\ell_{\max}$, each player $p_i$ would propose to the arm with its distinct index. And each player can then observe that $\abs{ \cup_{i'\in[N]} \set{\bar{A}_{i'}(t)} }= N$. 
Based on this observation, all players would enter in the DA phase simultaneously at the end of the $\ell_{\max}$-th epoch. 
\end{proof}

\begin{lemma}\label{lem:DA:steps}

The offline DA algorithm stops in at most $\min\set{N^2, NK}$ steps. 
And the player-optimal stable arm of each player is the first $\min\set{N,K}$-ranked in its preference list. 
\end{lemma}
\begin{proof}
According to the offline DA algorithm procedure, once an arm has been proposed by players, this arm has a temporary partner.
Above all, once $N$ arms have been proposed, they will occupy $N$ players and the algorithm stops. So before the algorithm stops, at most $N-1$ arms have been previously proposed. 
Since players propose to arms one by one according to their preference list, a player can only be rejected by an arm at most once. Thus $N-1$ arms can reject at most $N$ players. The worst case is that one rejection happens at one step, resulting in the $N^2$ total time complexity. 
And since there are at most $K$ arms, the DA algorithm would stop in $\min\set{N^2, NK}$ steps. 

And since only $\min\set{N,K}$ arms have been proposed at the end, the final matched arm of each player must belong to the first $\min\set{N,K}$-ranked in its preference list. 

\end{proof}

\begin{lemma}\label{lem:cen:ucblcb}
Conditional on $\urcorner \cF$, $\ucb_{i,j}(t)<\lcb_{i,j'}(t)$ implies $\mu_{i,j}<\mu_{i,j'}$ for any time $t$. 
\end{lemma}

\begin{proof}
Conditional on $\urcorner \cF$, for each $i\in[N],j\in[K]$, we have
\begin{align*}
 \lcb_{i,j}(t) =\hat{\mu}_{i,j}(t)-\sqrt{\frac{6\log T}{ T_{i,j}(t)}}  \le \mu_{i,j} \le \hat{\mu}_{i,j}(t)+\sqrt{\frac{6\log T}{ T_{i,j}(t)}} = \ucb_{i,j}(t)\,.
\end{align*}
Thus if $\ucb_{i,j}(t)<\lcb_{i,j'}(t)$, there would be 
\begin{align*}
  \mu_{i,j} \le  \ucb_{i,j}(t) <\lcb_{i,j'}(t) \le \mu_{i,j'}\,.
\end{align*}
\end{proof}

% \fang{proof modify}
\begin{lemma}\label{lem:cen:pulltime}
% Let $T_i(t) = \min\set{T_{i,j}(t): j \in S_i(t)}$, $\bar{T}_i = \frac{96\log T}{\Delta^2}$. Conditional on $\urcorner \cF$, if $T_i(t) > \bar{T}_i$, we have $\ucb_{i,j}(t)< \lcb_{i,j'}(t)$ for any $j,j'\in S_{i}(t)$ with $\mu_{i,j}<\mu_{i,j'}$.
Consider the player $p_i$ and two arms $a_j$ and $a_{j'}$ with $\mu_{i,j}<\mu_{i,j'}$. 
Conditional on $\urcorner \cF$, if $\min\set{T_{i,j}(t), T_{i,j'}(t) } > \frac{96\log T}{\Delta_{i,j,j'}^2}$, we have $\ucb_{i,j}(t)< \lcb_{i,j'}(t)$.  
\end{lemma} 

\begin{proof}
By contradiction, suppose $\ucb_{i,j}(t)\ge \lcb_{i,j'}(t)$. Conditional on $\urcorner \cF$, we have 
\begin{align*}
 \mu_{i,j'}- 2\sqrt{\frac{6\log T}{{T}_{i}(t)}} \le \lcb_{i,j'}(t) \le \ucb_{i,j}(t) \le \mu_{i,j}+2\sqrt{\frac{6\log T}{{T}_{i}(t)}} \,.
\end{align*}
We can then conclude $\Delta_{i,j,j'} \le 4 \sqrt{\frac{6\log T}{\min\set{T_{i,j}(t),T_{i,j'}(t) }}}$ and thus $\min\set{T_{i,j}(t),T_{i,j'}(t) } \le \frac{96 \log T}{\Delta_{i,j,j'}^2}$, which contradicts the fact that $\min\set{T_{i,j}(t),T_{i,j'}(t) } > \frac{96 \log T}{\Delta_{i,j,j'}^2}$. 
\end{proof}

\section{Analysis of The AETDA Algorithm (Algorithm \ref{alg:AETDA})}
\label{sec:proof:aetda}

\subsection{Proof of Theorem \ref{thm:aetgs:regret}}

    The player-optimal stable regret of each player $p_i$ by following our AETDA algorithm (Algorithm \ref{alg:AETDA}) satisfies
% \begin{align}
%     \overline{R}_i(T) &=\EE{\sum_{t=1}^T \bracket{\mu_{i,\overline{m}_i} - X_{i}(t)} } \notag \\
%     &\le \EE{\sum_{t=1}^T \bracket{\mu_{i,\overline{m}_i} - X_{i}(t)} \mid 
%     \urcorner \cF } + T \cdot \PP{\cF } \cdot \mu_{i,\overline{m}_i}\notag \\
%     &\le \EE{\sum_{t=1}^T \sum_{a_j\in \cK} \bOne{\bar{A}_i(t)=a_j} \Delta_{i,\overline{m}_i,j} \mid \urcorner \cF } + \EE{\sum_{t=1}^T \bOne{\bar{A}_i(t)=\emptyset} \mu_{i,\overline{m}_i} \mid \urcorner \cF } + 2NK \mu_{i,\overline{m}_i} \notag \\
%     &\le \frac{192\min\set{N^2, NK} C\log T}{\Delta^2}\cdot \mu_{i,\overline{m}_i}  + 2NK \mu_{i,\overline{m}_i}  \label{eq:aetda:regret}  \\
%     &= O\bracket{N\min\set{N,K}C\log T/\Delta^2} \notag \,,
% \end{align}
\begin{align}
    \overline{R}_i(T) &=\EE{\sum_{t=1}^T \bracket{\mu_{i,\overline{m}_i} - X_{i}(t)} } \notag \\
    &\le \EE{\sum_{t=1}^T \bracket{\mu_{i,\overline{m}_i} - X_{i}(t)} \mid 
    \urcorner \cF } + T \cdot \PP{\cF } \cdot \mu_{i,\overline{m}_i}\notag \\
    &\le \EE{\sum_{t=1}^T \bOne{\opt_i \neq -1}\bracket{\mu_{i,\overline{m}_i} - X_{i}(t)} \mid 
    \urcorner \cF } + \EE{\sum_{t=1}^T \bOne{\opt_i = -1}\bracket{\mu_{i,\overline{m}_i} - X_{i}(t)} \mid 
    \urcorner \cF } + T \cdot \PP{\cF } \cdot \mu_{i,\overline{m}_i}\notag \\
    % &\le \EE{\sum_{t=1}^T \sum_{a_j\in \cK} \bOne{\bar{A}_i(t)=a_j} \Delta_{i,\overline{m}_i,j} \mid \urcorner \cF } + \EE{\sum_{t=1}^T \bOne{\bar{A}_i(t)=\emptyset} \mu_{i,\overline{m}_i} \mid \urcorner \cF } + 2NK \mu_{i,\overline{m}_i} \notag \\
    &\le \frac{192\min\set{N^2, NK} C\log T}{\Delta_{\overline{m}}^2}\cdot \mu_{i,\overline{m}_i}  + 2NK \mu_{i,\overline{m}_i}  \label{eq:aetda:regret}  \\
    &= O\bracket{N\min\set{N,K}C\log T/\Delta_{\overline{m}}^2} \notag \,,
\end{align}
where Eq. \eqref{eq:aetda:regret} comes from Lemma \ref{lem:aetda:collision} and \ref{lem:aetda:sub-optimal}.

\begin{lemma}\label{lem:aetda:collision}
    Following the AETDA algorithm, conditional on $\urcorner \cF$,  the regret of each player $p_i$ suffered when focusing on arms satisfies that 
    \begin{align*}
        % \EE{\sum_{t=1}^T \bOne{\bar{A}_i(t)=\emptyset} \mu_{i,\overline{m}_i} \mid \urcorner \cF } 
        \EE{\sum_{t=1}^T \bOne{\opt_i \neq -1}\bracket{\mu_{i,\overline{m}_i} - X_{i}(t)} \mid 
    \urcorner \cF } \le  \frac{96\min\set{N^2, NK} C\log T}{\Delta_{\overline{m}}^2}\cdot \mu_{i,\overline{m}_i} \,.
    \end{align*}
\end{lemma}

\begin{proof}
Recall that conditional on $\urcorner \cF$ and Lemma \ref{lem:cen:ucblcb}, the AETDA algorithm is an online adaptive version of the offline DA algorithm and it will reach the player-optimal stable matching. 
Once $p_i$ focuses on an arm ($\opt_i\neq -1$), this arm must have a higher ranking than the player-optimal stable one.
So the regret in this part only happens when $p_i$ collides with others at arm $\opt_i$. 

Lemma \ref{lem:DA:steps} shows that the offline DA algorithm proceeds in at most $\min\set{N^2, NK}$ steps. Denote $t_s$ as the round index of the start of step $s$ in our AETDA. Then the regret caused when focusing on arms can be decomposed into these steps as Eq. \eqref{eq:aetda:explore:steps}. 
The total regret in this part satisfies
\begin{align}
        & \EE{\sum_{t=1}^T \bOne{\opt_i \neq -1}\bracket{\mu_{i,\overline{m}_i} - X_{i}(t)} \mid 
    \urcorner \cF } \notag \\ \le& \EE{\sum_{s=1}^{\min\set{N^2, NK}} \sum_{t=t_{s}}^{t_{s+1}-1} \bOne{\opt_i \neq -1, \bar{A}_i(t)=\emptyset} \mu_{i,\overline{m}_i} \mid \urcorner \cF } \label{eq:aetda:explore:steps} \\
        \le & \sum_{s=1}^{\min\set{N^2, NK}} \frac{96C\log T}{\Delta_{\overline{m}}^2}\cdot \mu_{i,\overline{m}_i} \label{eq:aetgs:explore:steps:times} \\
        \le& \frac{96\min\set{N^2, NK} C\log T}{\Delta_{\overline{m}}^2}\cdot \mu_{i,\overline{m}_i} \notag \,. 
    \end{align}
In each step, the regret occurs when $p_i$ focuses on the arm $\opt_i$ and other players round-robin explore this arm who is preferred more by $\opt_i$. Based on Lemma \ref{lem:cen:pulltime}, an arm is explored for at most $96\log T/\Delta_{\overline{m}}^2$ times by another player $p_{i'}$ before $p_{i'}$ identifies its currently most preferred arm (which has a higher ranking than the player-optimal stable arm of $p_{i'}$).  
And when $N$ players explore $K$ arms, at most $C$ rounds are required to ensure each player can be matched with each arm once. 
That is why Eq. \eqref{eq:aetgs:explore:steps:times} holds. 
\end{proof}

\begin{lemma}\label{lem:aetda:sub-optimal}
    Following the AETDA algorithm, the regret of each player $p_i$ caused by exploring sub-optimal arms satisfies that 
    \begin{align*}
        % \EE{\sum_{t=1}^T \sum_{a_j\in \cK} \bOne{\bar{A}_i(t)=a_j} \Delta_{i,\overline{m}_i,j} \mid \urcorner \cF } \le \frac{96\min\set{N,K} C\log T}{\Delta^2}\cdot  \mu_{i,\overline{m}_i} \,.\\
        \EE{\sum_{t=1}^T \bOne{\opt_i = -1}\bracket{\mu_{i,\overline{m}_i} - X_{i}(t)} \mid 
    \urcorner \cF } \le \frac{96\min\set{N,K} C\log T}{\Delta_{\overline{m}}^2}\cdot  \mu_{i,\overline{m}_i} \,.\\
    \end{align*}
\end{lemma}

\begin{proof}
Recall that $\opt_i=-1$ means that player $p_i$ explores to find its most preferred available arm. 
 % So the regret only occurs when $p_i$ explores arms in a round-robin manner. 
According to Lemma \ref{lem:DA:steps}, the player-optimal stable arm must be the first $\min\set{N,K}$ ranked, denote $t_{s,s}$ and $t_{s,e}$ as the start and end round index when $p_i$ explores to find the $s$-ranked arm, then the regret can be decomposed as Eq. \eqref{eq:aetgs:dueto:GS:explore}. The total regret caused by exploring sub-optimal arms satisfies that
    \begin{align}
        &\EE{\sum_{t=1}^T \bOne{\opt_i = -1}\bracket{\mu_{i,\overline{m}_i} - X_{i}(t)} \mid 
    \urcorner \cF } \notag \\
        \le &  \EE{ \sum_{s=1}^{\min\set{N,K}} \sum_{t=t_{s,s}}^{t_{s,e}}  \bracket{\mu_{i,\overline{m}_i} - X_{i}(t)} \mid \urcorner \cF } \label{eq:aetgs:dueto:GS:explore} \\
        \le& \sum_{s=1}^{\min\set{N,K}} \frac{96C\log T}{\Delta_{\overline{m}}^2} \cdot \mu_{i,\overline{m}_i} \label{eq:aetgs:dueto:GS:explore:times} \\
        \le& \frac{96\min\set{N,K} C\log T}{\Delta_{\overline{m}}^2}\cdot  \mu_{i,\overline{m}_i} \notag \,,
    \end{align}
    where Eq. \eqref{eq:aetgs:dueto:GS:explore:times} holds based on Lemma \ref{lem:cen:pulltime} and the fact that each player can match each arm once in at most $C$ rounds during round-robin exploration.
\end{proof}

\subsection{Proof of Theorem \ref{thm:aetgs:strategic}}

    For the offline DA algorithm, it has been shown that when all of the other players submit their true rankings, no single player can improve its final matched partner by misreporting its preference ranking \cite{roth1982economics,dubins1981machiavelli}. 

Recall that our algorithm is an adaptive online version of the GS algorithm and $\mathrm{opt}_i$ represents the estimated most preferred arm of player $p_i$ in the currently available arm set $S_i$.
    There are mainly two cases of misreporting. One is that $p_i$ wrongly reports an arm as its estimated optimal one which actually is not. And the other case is that $p_i$ has learned the optimal arm but reports $\mathrm{opt}_i$ as $-1$. 
     According to the property of the DA algorithm, no matter whether $p_i$ has estimated well its current most preferred arm, reporting a wrong one would finally result in a less-preferred arm. 
    And on the other hand, if $p_i$ has already estimated well its most preferred arm, misreporting $\mathrm{opt}_i=-1$ would keep it in the round-robin exploration process. 
    According to the property of GS, no matter whether all players enter the algorithm simultaneously, their final matched arm is always the player-optimal one. 
    So misreporting $\mathrm{opt}_i=-1$ is equivalent to the player delaying entry into the offline DA algorithm and the final matching would not change.

%!TEX root =  main.tex

\section{Proof of Theorem \ref{thm:decen} }\label{sec:proof:decen}

We first provide a proof sketch of Theorem \ref{thm:decen} and the detailed proof is presented later.

\paragraph{Proof Sketch}\label{sec:proof:sketch}
We first show that, with high probability, the real preference value $\mu_{i,j}$ can be upper bounded by $\ucb_{i,j}(t)$ and lower bounded by $\lcb_{i,j}(t)$ in each round $t$. In the following, we would analyze the algorithm based on this high-probability event. 

At a high level, Algorithm \ref{alg:decen} can be regarded as an online version of DA with the arm side proposing which returns the player-pessimal stable matching. 
Specifically, at each step of the DA algorithm with the arm side proposing, each arm $a_j$ proposes to the player set $\ch_j(P_{i,j})$, which is equivalent in our algorithm to each player $p_i$ proposing arms in the plausible set constructed as $S_i(t)=\set{a_j\in \cK: p_i \in \ch_j(P_{i,j})}$.
Then each player would reject all but the most preferred arm among those who propose to it, equivalent in our algorithm to players deleting all arms in the plausible set but the one with the highest preference value. 
But since players do not know their own preferences in our setting, they need to explore these arms to learn the corresponding preference values. 
Based on the above high-probability event and the construction of the two confidence bounds, if $\mu_{i,j}<\mu_{i,j'}$ for player $p_i$ and arms $a_j,a_{j'}$ in its plausible set, these two arms would be selected by $p_i$ for at most $O(\log T/\Delta^2_{i,j,j'})$ times before $\ucb_{i,j}<\lcb_{i,j'}$ and further arm $a_j$ is considered to be less preferred than other plausible arms. 
We can regard this event as $p_i$ rejects arm $a_j$ in DA. 
When all players determine the most-preferred arm from the plausible set, the corresponding DA can proceed to the next step and arms then propose the preferred subset of players among those who have not rejected them. 
In the offline DA algorithm, the rejection can happen for at most $NK$ times since each player can reject each arm at most once. 
And all arms that are rejected by each player $p_i$ in the DA algorithm with arm side proposing are less preferred than the player-pessimal stable arm of $p_i$. 
Correspondingly, the regret of our algorithm is at most $O(NK\log T/\Delta_{\underline{m}}^2)$ before reaching stability.

\subsubsection{Full Proof}\label{sec:proof:full}

In this section, we provide the detailed proof of Theorem \ref{thm:decen}.

Let $P_{i,j}(t)$ be the value of $P_{i,j}$ at the end of round $t$.
Recall $\bar{A}(t)=\set{(p_i,\bar{A}_i(t)):p_i\in \cN}$ is the matching at round $t$ and $M^*$ is the set of all stable matchings. Further, denote $A(t)=\set{(p_i,A_i(t)):p_i\in \cN}$ as the set of players and their selected arms at round $t$.  
The player-pessimal stable regret of $p_i$ can then be bounded by 
\begin{align}
    \underline{R}_i(T) &\le  \EE{\sum_{t=1}^T \bOne{\bar{A}(t) \notin M^* }}\cdot \mu_{i,\underline{m}_i} \notag\\
    &= \EE{\sum_{t=1}^T \bOne{{A}(t)  \notin M^* }}\cdot \mu_{i,\underline{m}_i} \label{eq:decen:nocollision} \\
    &\le \EE{\sum_{t=1}^T \bOne{{A}(t) \notin M^*  } \mid \urcorner \cF }\cdot \mu_{i,\underline{m}_i} + T\cdot \PP{ \cF}\cdot \mu_{i,\underline{m}_i} \notag \\
    &\le \bracket{\frac{192NK\log T}{\Delta_{\underline{m}}^2} + 2NK} \mu_{i,\underline{m}_i} + 2NK\mu_{i,\underline{m}_i} \label{eq:decen:main and badevent}  \\
    &= O\bracket{NK\log T/\Delta_{\underline{m}}^2} \notag \,,
\end{align}

where Eq.\eqref{eq:decen:nocollision} holds according to Lemma \ref{lem:decen:nocollision}, Eq.\eqref{eq:decen:main and badevent} comes from Lemma \ref{lem:cen:badevent} and Lemma \ref{lem:decen:mainevent}.

\begin{lemma}\label{lem:decen:nocollision}
In Algorithm \ref{alg:decen}, at each round $t$, $\bar{A}_i(t) = A_i(t)$ for each player $p_i$. 
\end{lemma} 
\begin{proof}
The case where $A_i(t) = \emptyset$ holds trivially. In the following, we mainly consider the case where $A_i(t) \neq \emptyset$. 

According to Lemma \ref{lem:decen:sameP}, all players have the same $P_{i,j}(t)$ at each time $t$ for each arm $a_j$. For simplicity, we then set $P_j(t) = P_{i,j}(t)$ for any arm $a_j$ and $p_i \in \cN$. In Algorithm \ref{alg:decen}, when player $p_i$ proposes to $A_i(t) =a_j \in S_i(t)$, we have $p_i \in \ch_j(P_j(t-1))$. 
Thus it holds that $A^{-1}_j(t)\subseteq \ch_j(P_j(t-1))$. 
According to the substitutability, for each player $p_i$ who proposes to $a_j$, $p_i \in \ch_j(P_j(t-1)\cap A^{-1}_j(t)) = \ch_j(A^{-1}_j(t))$. According to the acceptance protocol of the arm side, each $p_i \in A^{-1}_j(t)$ can be successfully accepted and $\bar{A}_i(t)=A_i(t)=a_j$ holds. 
\end{proof}

\begin{lemma}\label{lem:decen:mainevent}
In Algorithm \ref{alg:decen}, for each player $p_i$, 
\begin{align*}
\EE{\sum_{t=1}^T \bOne{{A}_i(t) \notin M^*}\mid \urcorner \cF } \le \frac{192NK\log T}{\Delta_{\underline{m}}^2} + 2NK \,.
\end{align*}
\end{lemma}

\begin{proof}
Recall that our Algorithm \ref{alg:decen} can be regarded as an online version of DA algorithm. 
At step $\ell$ of DA, define $S_{i,\ell}$ as the set of arms who propose player $p_i$ and $R_{i,\ell}$ as the set of arms rejected by $p_i$. It is straightforward that $|S_{i,\ell}|=|R_{i,\ell}|+1$ since each player only accepts one arm among those who propose to it and rejects others. 
Since DA stops when no rejection happens, we have $\max_{i\in[N]}|R_{i,\ell}|\ge 1$ for each step $\ell$ before DA stops.

\begin{figure}[tbh!]
\centering
\hspace{-0.6cm}
    \begin{tikzpicture}[scale=0.8, every node/.style={node distance = 0.9cm}]

    \draw [->](5,0)--(22.5,0);
    \coordinate [label=round $t$](round) at(22.5,-0.8);

    \draw (5,1)--(5,-1);
    
    \draw [->](8.5,0.5)--(9,0.5);
    \draw [->](5.5,0.5)--(5,0.5);
    \coordinate [label=step $1$ of DA](s1) at(7,0.2);
    \coordinate [label=length $L_1$](s1R) at(7,-1.2);
    
    \draw (9,1)--(9,-1);
    
    \draw [->](10,0.5)--(9,0.5);
    \draw [->](14,0.5)--(15,0.5);
    \coordinate [label=step $2$ of DA](s2) at(12,0.2);
    \coordinate [label=length $L_2$](s2R) at(12,-1.2);
    
    \draw (15,1)--(15,-1);
    
    \coordinate [label=$\cdots$](cdots) at(15.5,0.2);
    \draw (16,1)--(16,-1);

    \draw [->](17,0.5)--(16,0.5);
    \draw [->](20,0.5)--(21,0.5);
    \coordinate [label=step $\tau$ of DA](stau) at(18.5,0.2);
    \coordinate [label=length $L_{\tau}$](stauR) at(18.5,-1.2);

    \draw (21,1)--(21,-1.2);
    
    \coordinate [label=reach stability](end) at(21,-2);

    \end{tikzpicture}

    \caption{A demonstration for the total horizon of Algorithm \ref{alg:decen}. The length $L_{\ell}$ of each step $\ell$ is $\max_{i\in[N]} 96{|S_{i,\ell}|\log T}/{\Delta_{\underline{m}}^2}+2$, where $S_{i,\ell}$ denotes the set of arms who propose player $p_i$ at step $\ell$ following the offline DA algorithm.  
    }
    \label{fig:illustration}
\end{figure}
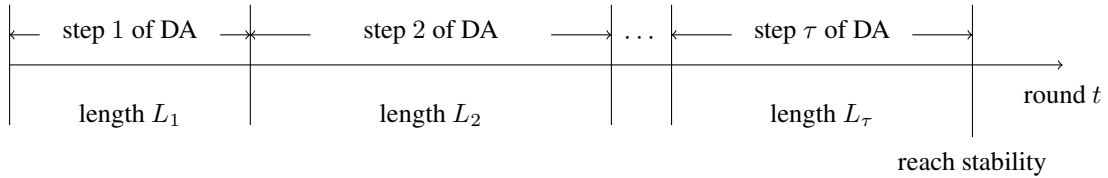

The total horizon $T$ in Algorithm \ref{alg:decen} can then be divided into several steps according to the DA algorithm. At each step $\ell$, each player $p_i$ attempts to pull the arm in $S_{i,\ell}$ in a round-robin way until it identifies the most preferred one. 
According to Lemma \ref{lem:cen:ucblcb}, once an arm is deleted from the plausible set, then it is truly less preferred. 
Further, based on Lemma \ref{lem:cen:pulltime} and the fact that all deleted arms are less preferred than the pessimal stable arm of each player, each step $\ell$ would last for at most $\max_{i\in[N]}96|S_{i,\ell}|\log T/\Delta_{\underline{m}}^2+2$ rounds, where the $2$ rounds are the time it takes for all players to detect the end of a step.   
Figure \ref{fig:illustration} gives an illustration for the total horizon of Algorithm \ref{alg:decen}. Formally, the regret can be decomposed as

\begin{align}
    \EE{\sum_{t=1}^T \bOne{{A}_i(t) \notin M^*}\mid \urcorner \cF } 
    &\le \sum_{\ell=1}^\tau \bracket{\max_{i\in[N]}|S_{i,\ell}|\cdot \frac{96\log T}{\Delta_{\underline{m}}^2} +2 }\label{eq:decen:duetofigure}\\
    &= \sum_{\ell=1}^\tau \bracket{\max_{i\in[N]}(|R_{i,\ell}|+1)\cdot \frac{96\log T}{\Delta_{\underline{m}}^2}+2} \notag \\
    &\le 2\sum_{\ell=1}^\tau \max_{i\in[N]}|R_{i,\ell}|\cdot \frac{96\log T}{\Delta_{\underline{m}}^2} +2NK \label{eq:decen:duetoRmorethan1}\\
    &\le 2\sum_{\ell=1}^\tau \sum_{i\in[N]}|R_{i,\ell}|\cdot \frac{96\log T}{\Delta_{\underline{m}}^2} +2NK \notag \\
    &\le \frac{192NK\log T}{\Delta_{\underline{m}}^2} +2NK  \label{eq:decen:duetoRejectionAtMostNK}\,,
\end{align}
where Eq.\eqref{eq:decen:duetofigure} holds according to Lemma \ref{lem:cen:pulltime} and Figure \ref{fig:illustration}, Eq.\eqref{eq:decen:duetoRmorethan1} holds since $\max_{i}|R_{i,\ell}|\ge 1$ before the offline DA stops and $\tau\le NK$ as at each step at least one rejection happens (thus DA lasts for at most $NK$ steps before finding the stable matching), Eq.\eqref{eq:decen:duetoRejectionAtMostNK} holds since the number of all rejections is at most $NK$.

\end{proof}

\begin{lemma}\label{lem:decen:sameP}
In Algorithm \ref{alg:decen}, for any arm $a_j\in \cK$ and round $t$, $P_{i,j}(t) = P_{i',j}(t)$ for any different players $p_i,p_{i'}$. 
\end{lemma}
\begin{proof} 

At the beginning, each player $p_i$ initializes $P_{i,j}=\cN$, thus the result holds. In the following rounds, player $p_i$ updates $P_{i,j}(t)$ only if it observes all players select the same arm for two consecutive rounds. Since the observations of all players are the same, they would update $P_{i,j}$ simultaneously. Above all, $P_{i,j}(t) = P_{i',j}(t)$ would always hold for any different player $p_i,p_{i'}$, arm $a_j$ and round $t$.

\end{proof}

% \subsection{Proof of Theorem \ref{thm:decen:strategy}}

% \begin{proof}[Proof of Theorem \ref{thm:decen:strategy}]
% According to the construction rule, $S_i$ is defined as the set of arms that can successfully accept player $p_i$ at the current round and still have the potential to be the most preferred one. 
% So for any arm $a_j \notin S_i$, there must be $p_i \notin \ch_j(P_{i,j})$.  
% This means that $p_i$ may be rejected and receive neither observation or reward when selecting $a_j$. So $p_i$ has no incentive to select arms beyond $S_i$. 

% Recall that our ODA algorithm is an online version of the DA algorithm with the arm-side proposing. \citet[Theorem 3]{vaish2017manipulating} show that when a single player $p_i$ misreports an optimal manipulation as its preference ranking, i.e., under which manipulation the player can match an arm that has a higher preference ranking than that under any other manipulation by following DA, then the resulting matching of DA is still a stable matching. Since the original matching is the players' least preferred one, each player can match an arm in this new matching that is better than the arm in the original matching generated under the true preference ranking. 

% \end{proof}

\section{Technical Lemma}

\begin{lemma}{(Corollary 5.5 in \cite{lattimore2020bandit})}\label{lem:chernoff}
Assume that $X_1, X_2,\ldots, X_n$ are independent, $\sigma$-subgaussian random variables centered around $\mu$. Then for any $\varepsilon > 0$,
\begin{align*}
    \PP{ \frac{1}{n} \sum_{i=1}^n X_i \ge  \mu + \varepsilon} \le \exp\bracket{-\frac{n\varepsilon^2}{2\sigma^2}}\,, \ \ \ \PP{ \frac{1}{n} \sum_{i=1}^n X_i \le  \mu - \varepsilon} \le \exp\bracket{-\frac{n\varepsilon^2}{2\sigma^2}}\,.
\end{align*}
\end{lemma}

\end{document}